\documentclass{article}
\usepackage[utf8]{inputenc}
\usepackage{amsmath}
\usepackage{amsfonts}
\usepackage{subfigure}
\usepackage{verbatim}
\usepackage{amssymb}
\usepackage{natbib}
\usepackage{fullpage}
\usepackage{color}

\usepackage{mathrsfs}
\usepackage{graphicx}
\usepackage{amsthm}
\usepackage{hyperref}
\usepackage{booktabs}
\usepackage{verbatim}
\usepackage{microtype}
\usepackage{bm}
\usepackage{algorithm,algorithmic}

\usepackage{hyperref}

\newtheorem{theorem}{Theorem}
\newtheorem*{theorem*}{Theorem}
\newtheorem{lemma}{Lemma}

\newtheorem*{lemma*}{Lemma}

\newtheorem{definition}{Definition}
\newtheorem{assumption}{Assumption}

\newcommand{\E}{\mathbb{E}}

\title{Extracting Latent State Representations \\
           with Linear Dynamics from Rich Observations}
\author{Abraham Frandsen\thanks{Duke University} \and Rong Ge\thanks{Duke University}}

\begin{document}
\maketitle

\begin{abstract}
Recently, many reinforcement learning techniques were shown to have provable guarantees in the simple case of linear dynamics, especially in problems like linear quadratic regulators. However, in practice, many reinforcement learning problems try to learn a policy directly from rich, high dimensional representations such as images. Even if there is an underlying dynamics that is linear in the correct latent representations (such as position and velocity), the rich representation is likely to be nonlinear and can contain irrelevant features. In this work we study a model where there is a hidden linear subspace in which the dynamics is linear. For such a model we give an efficient algorithm for extracting the linear subspace with linear dynamics. We then extend our idea to extracting a nonlinear mapping, and empirically verify the effectiveness of our approach in simple settings with rich observations.
\end{abstract}

\section{Introduction}



Reinforcement learning has made tremendous progress recently, achieving strong performance in difficult problems like go \citep{silver2017mastering} and Starcraft \citep{vinyals2019alphastar}. 
 A common theme in the recent progress is the use of neural networks to handle the cases when the system dynamics and policy are both highly nonlinear. However, theoretical understanding for reinforcement learning has been mostly limited to the tabular setting (where the number of state/actions is small) or when the underlying dynamics of the system is linear (see Section~\ref{sec:related}).

Requiring the dynamics to be linear is especially limiting for problems with {\em rich}, high dimensional output, e.g. manipulating a robot from video frames or playing a game by observing pixel representations. Consider a simple system where we control an object by applying forces to it: the state of the object (position and velocity) can actually be linear according to physical laws. However, if our observation is a rendering of this object in a 3-d environment, the observation contains a lot of redundant information and is not going to have linear dynamics. Such problems can potentially be solved by learning a {\em state representation mapping} $\phi$ that maps the complicated observation to states that satisfy simpler dynamics. State representation learning is popular in practice, see the survey by \citet{lesort2018state} and more references in Section~\ref{sec:related}. We are interested in the question of when we can provably extract a state representation that encodes linear dynamics.

We first consider a simple theoretical model where the full observation $x$ does not have linear dynamics, but there exists an unknown subspace $V$ where the projection $y = P_V x$ has linear dynamics. This corresponds to the case when the state representation mapping is a linear projection. We use the ideas of learning {\em inverse models}, which try to predict the action that was used to transition from one state to another. We show that a convex relaxation of the inverse model objective provably learns the unknown subspace $V$. 

Of course, in more complicated settings one might need a nonlinear mapping in order to extract a latent space representation that has linear dynamics. In particular, one can think of the last layer of a neural network as providing a nonlinear state representation mapping. 
We extend our approach to the nonlinear setting, and show that if we can find a solution to a similar nonconvex optimization problem with 0 loss, then the representation will have nontrivial linear dynamics. 

In the remainder of the paper, we first discuss related works in Section~\ref{sec:related}. We then introduce our model and discuss how one can formalize learning a state representation mapping as an optimization problem in Section~\ref{sec:model_intro}. In Section~\ref{sec:model_detail} we first consider the case of a linear state representation mapping, and prove that our algorithm indeed recovers the underlying state representation. We then extend the model to nonlinear state representation mapping. Finally in Section~\ref{sec:experiments} we show our approach can be used to learn low dimensional state representations for simple RL environments.
\section{Related Work}
\label{sec:related}

\paragraph{State Representation Learning with Rich Observations}	
Several recent papers have addressed the problem of state representation learning (SRL) in control problems.
\citet{lesort2018state} survey the recent literature and identify four categories that describe many SRL approaches: reconstructing the observation, learning a forward dynamics model, learning an inverse dynamics model, and using prior knowledge to constrain the state space.
\citet{raffin2019decoupling} evaluate many of these SRL approaches on robotics tasks and show how to combine the strengths of the different methods.
Several papers adopt the tactic of learning inverse models \citep{pathak2017curiosity,zhang2018decoupling, shelhamer2016loss} and demonstrate its effectiveness in practice, but they lack a theoretical analysis of the approach. Our work aims to help fill this gap.

Common domains with rich observations include raw images or video frames from video games \citep{anand2019unsupervised, pathak2017curiosity},
robotics environments \citep{higgins2017darla, finn2016deep}, and renderings of classic control problems \citep{watter2015embed, van2016stable}, and deep learning methods have enabled success in this space.
Recently, \citet{ha2018world} introduced a method for learning ``world models'' which learn low-dimensional representations and dynamics in an unsupervised manner, enabling simple linear policies to achieve effective control. 
\citet{hafner2019dream} utilize world models in conjunction with latent imagination to learn behaviors that achieve high performance in terms of reward and sample-efficiency on several visual control tasks. 

On the theoretical side, \citet{du2019good} investigate whether good representations lead to sample-efficient reinforcement learning in the context of MDPs, showing exponential lower bounds in many settings.  
Our setting circumvents these negative results because the representations that we learn transform the nonlinear problem into a linear (and hence tractable) control problem. 
Other recent works \citep{du2019provably, misra2019kinematic} study the Block MDP model,  in which the action space is finite and the potentially infinite, high-dimensional observation space is generated from a finite set of latent states.
Our model, by contrast, considers continuous state and action spaces, and we assume the existence of a latent \emph{subspace} that linearly encodes the relevant dynamics.  

\paragraph{Linear Dynamical Systems and Control Problems}
Linear dynamical systems and control problems have been extensively studied for many decades and admit 
efficient, robust, and provably correct algorithms. 
We mention a few recent theoretical developments in this area.
For the problem of system identification, \citet{qin2006overview} gives a review of subspace identification methods; the approach we develop in this work, while for a different setting, is somewhat similar to the regression approaches described in the review.
Other recent works  analyze gradient-based methods for system identification \citep{hardt2018gradient}, policy optimization \citep{fazel2018global}, and online control \citep{cohen2018online} in the setting of linear dynamical systems and quadratic costs.

\paragraph{Koopman Operator Theory}
Many recent works have approached the problem of linearizing the dynamics by finding a \emph{lifting} based on the Koopman operator (an infinite-dimensional linear operator representation). 
\citet{williams2015data} propose the extended dynamic mode decomposition algorithm to approximate the leading eigenfunctions of the Koopman operator, and variants of this idea include the use of kernels \citep{kawahara2016dynamic}, dictionary learning \citep{li2017extended}, and neural networks \citep{lusch2018deep, yeung2019learning}, as well as extensions to  control problems \citep{folkestad2019extended}.
The use of the Koopman operator is not appropriate in our setting with high-dimensional rich observations, since we are interested in finding a low-dimensional mapping that linearizes only a subspace of the original system.

\section{Hidden Subspace Model and State Representation Learning}
\label{sec:model_intro}
\newcommand{\R}{\mathbb{R}}

In this section we first introduce a basic model which allows a linear state representation mapping. Later we show that the linear state representation can be learned efficiently.

We follow the notation of a discrete-time control system, where we use $x_t$ to denote the state at the $t$-th step, $u_t$ to denote the control signal (action) at the $t$-th step, and $f$ denotes the dynamics function, where $x_{t+1} = f(x_t, u_t).$
Note that here we assume the dynamics $f$ is deterministic, but we will later see that only some parts of $f$ need to be deterministic. 
A \emph{state representation mapping} is a function $\phi:\R^d\to \R^r$ that maps $x_t$ to a different space (usually $r\ll d$) such that the dynamics governing the evolution of $\phi(x_t)$ are simpler, e.g. linear.

\subsection{Hidden Subspace Model} We consider a model with a latent ground truth state representation $h_t \in \R^r$, which satisfies linear dynamics:
\[
h_{t+1} = \bar{A} h_{t} + \bar{B} u_{t}.  
\]
The high-dimensional observation $x_t$ is related to the latent state by $x_t = Vh_t + V^\perp g(h_t).$
Here, $V\in \R^{d\times r}$ denotes a basis for a subspace, and $V^\perp \in \R^{d\times (d-r)}$ is a basis for the orthogonal subspace, and $g: \R^r \to \R^{d-r}$ is an arbitrary nonlinear function that captures the redundant information in $x_t$. We will use $y_t$ to denote $Vh_t$ and $z_t$ to denote $V^\perp g(h_t)$. Note that $g$ does not need to be a function, it just needs to specify a conditional distribution $z_t|h_t$. The model is illustrated in Figure~\ref{fig:model}. 

\begin{figure}[t]
\vskip 0.2in
\begin{center}
\centerline{\includegraphics[width=.3\textwidth]{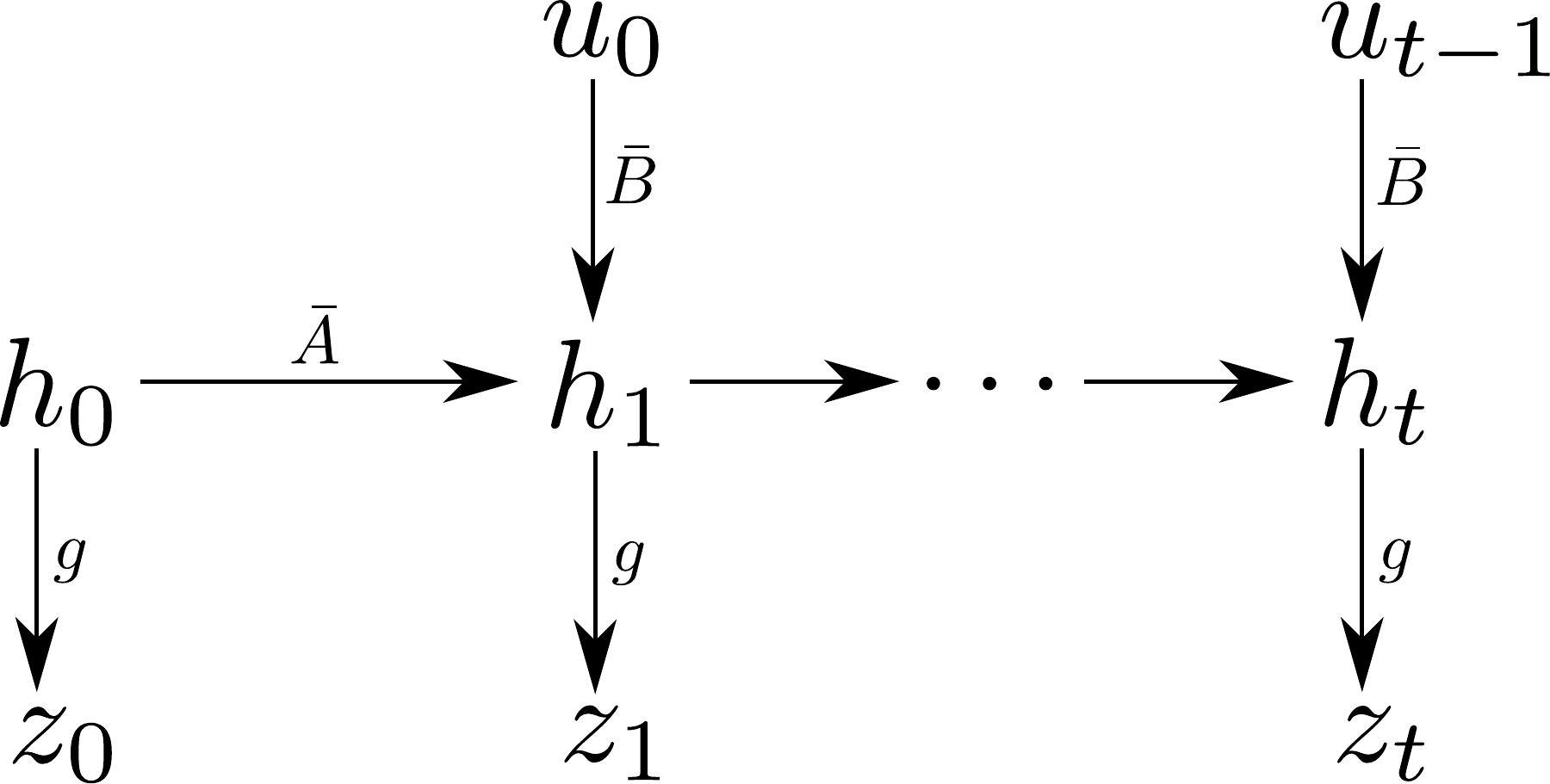}}
\caption{A graphical illustration of our hidden subspace model.
The hidden states $h_i$ evolve according to a linear control system 
and generate nonlinear features $z_i$.}
\label{fig:model}
\end{center}
\vskip -0.2in
\end{figure}

In this model, $y_t$ also satisfies linear dynamics, namely 
$y_{t+1} = V\bar{A}V^+ y_t + V\bar{B} u_{t-1} =: Ay_t + Bu_{t-1}.$
Here we define $A = V\bar{A}V^+$ and $B = V\bar{B}$. 
The ground truth mapping $\phi$ should map $x_t$ to $h_t$ (or any invertible linear transformation of $h_t$). 
To find this mapping it suffices to find the subspace $V$.


\subsection{Learning Linear State Representations}


To learn the state representation, we adopt the idea of learning an inverse model. Consider taking random actions $u_0,u_1,...,u_{i-1}$ from a random initial state $x_0$. We aim to predict the last action $u_{i-1}$, given $x_0$, $x_i$ and the previous actions. As the underlying process is linear, it makes sense to use a linear predictor $Px_i - L_i x_0 - \sum_{k=1}^{i-1} T_k u_{i - 1 - k}$. Therefore a natural loss function to consider is
\begin{equation}
\label{eq:linear_relaxation}
\underset{\theta}{\min} \frac{1}{2}\mathbb{E} \sum_{i=1}^r \| Px_i - L_ix_0 -\sum_{k=1}^{i-1}T_ku_{i-1-k} - u_{i-1}\|_2^2
\end{equation}
Here, $\theta$ is the tuple of parameters $(P,L_1,\ldots,L_r, T_1,\ldots,T_{r-1})$, and the expectation is taken over the randomness of $x_0, u_0, \ldots, u_{r-1}$.
As we will see later, the optimal solution for this loss function can be related to matrices $A$ and $B$. 

Note that this is a convex objective function. Our first results show that when $A, B$ are in general position, and $g$ is really nonlinear (which we formalize in Section~\ref{sec:model_detail}), one can recover the subspace $V$, and this extends to the finite sample setting.

\begin{theorem}[informal] When the parameters $A, B$ are in general position and the redundant information $z_t$ is not linearly correlated with $y_t$, for the minimum norm solution to objective function 
solution \eqref{eq:linear_relaxation}, the columns of every matrix are in $V$, and the union of the column spans is equal to $V$. In the finite sample setting, the same holds provided the number of samples is polynomial in $d, r, l$.
\end{theorem}

\subsection{Nonlinear State Representation}

To learn a nonlinear state representation mapping, we consider the setting where the observed state $x_t$ is first transformed by a feature mapping or kernel $\phi$ (e.g. polynomial or RBF features, or a hidden layer of a neural network), and some subset or linear combination of these new features have linear dynamics.
Again we try to predict the actions given starting and ending states:
\begin{equation}
\label{eq:nonlinear}
\underset{\theta}{\min} \frac{1}{2}\mathbb{E} \sum_{i=1}^\tau \| P\phi(x_i) - L_i\phi(x_0) -\sum_{k=1}^{i-1}T_ku_{i-1-k} - u_{i-1}\|_2^2
\end{equation}
Here the parameter $\theta$ includes the nonlinear function $\phi$. Of course, since $\phi$ is nonlinear, the optimization problem now is nonconvex. 
Our result shows that optimizing this objective for large enough $\tau$ allows us to find a linear dynamics:

\begin{theorem}[informal] If the objective function \eqref{eq:nonlinear} is optimized to have 0 loss, and the column span of $L_\tau^\top$ is a subspace of the column spans of previous $L_1^\top,L_2^\top,...,L_{\tau-1}^\top$, then one can extract a matrix $Q$ such that $Q\phi(x)$ satisfies a linear dynamics.
\end{theorem}
\section{Learning the Hidden Subspace Model}
\label{sec:model_detail}
In this section, we discuss our learning algorithm based on Equation~\ref{eq:linear_relaxation} in more detail. 
We show that this approach efficiently learns the linear state representation in our hidden subspace model when certain assumptions are satisfied.
We also study the sample complexity of this problem when we only have i.i.d. samples from the model. 
In the supplementary material we study a simplified version of the model where there is noise.

We first establish some notation and definitions.
If $a_1, \ldots, a_k$ are vectors, we let $(a_1, \ldots, a_k)$ denote the concatenation of these vectors.
For random vectors $a$ and $b$, let $\Sigma_{ab}$ denote the matrix $\E[ab^\top]$.
Let $\rho_{ab}$ denote the canonical correlation between $a$ and $b$, i.e.
\[
\rho(a,b) = \underset{a',b'}{\max} \frac{\E[\langle a,a'\rangle\langle b,b'\rangle]}{\sqrt{\E[\langle a, a'\rangle^2]\E[\langle b, b'\rangle^2]}}.
\]
For a matrix $A$, let $A^+$ denote its Moore-Penrose pseudoinverse, and let $\text{col}(A)$ denote the column-space of $A$. 
Let $\sigma_{min}(A)$ denote the smallest nonzero singular value of $A$.

We use the symbols $x_t, y_t, z_t$ to denote state vectors in $\mathbb{R}^d$ whose values change depending on the time index $t \in \mathbb{N}$.
We use the symbols $u_t$ to denote control input vectors in $\mathbb{R}^l$, where $l < d$.
In the context of our hidden subspace model, we let $V$ denote both the latent subspace itself, as well as a $d\times r$ matrix whose columns form a basis for the subspace. We refer to this subspace as the \emph{linearizing subspace}.

We adapt the standard notion of controllability from control systems theory to be subspace dependent.
\begin{definition}
Given matrices $A \in \mathbb{R}^{d\times d}$, $B \in \mathbb{R}^{d\times l}$, and an $r$-dimensional subspace $V \subset \mathbb{R}^d$, we say that the tuple $(A, B)$ is $V$-\emph{controllable} if $\text{col}(A), \text{col}(A^\top), \text{col}(B) \subset V$
and the $d\times rl$ matrix
\[
\begin{bmatrix}
B & AB & \cdots & A^{r-1}B
\end{bmatrix}
\]
has rank $r$.
\end{definition}

\subsection{Learning via Inverse Dynamics}
We now further motivate the optimization problem given in Equation~\eqref{eq:linear_relaxation}.
Since the dynamics on $V$ are linear and the process $z_t$ has a nonlinear dependence on the hidden state, it is reasonable to identify a linear relationship between the controls and the observed states -- concretely, we predict the most recent action as a linear function of past observed states and actions.
This problem is a form of \emph{learning an inverse model}, which is a well-established approach in the state representation learning literature as noted in Section~\ref{sec:related}.

Consider one step of the latent linear control system.
We have
\[
x_1 = y_1 + z_1= Ay_0 + Bu_0 + z_1= Ax_0 + Bu_0 + z_1.
\]
If $B$ has full column rank, then we have $B^+B = I$ and $B^+z_1 = 0$ since the rows of $B^+$ are in $V$.
Hence,  we can solve for the action $u_0$ given $x_0$ and $x_1$ as 
$u_0 = B^+x_1 - B^+Ax_0.$
This expression suggests that if we fit a linear model to predict $u_0$ given $x_0$ and $x_1$, the solution may allow us to recover
$B^+$ and $B^+A$, both of which reveal part of the latent subspace $V$.
Advancing the system up to timestep $i$, we have a similar relationship:
\[u_{i-1} = B^+x_i - B^+A^ix_0 - \sum_{k=1}^{i-1}B^+A^kBu_{i-1-k}.\]
Once again, if we fit a linear model to predict $u_{i-1}$ from $x_i, x_0, u_0,\ldots,u_{i-2}$, then we ought to be able to recover more of $V$.

Trying to solve for $A$ and $B$ directly by minimizing a squared error loss based on the above expression is problematic, given the presence of high powers of $A$ and products between $A$, $B$, and $B^+$. 
The optimization landscape corresponding to such an objective function is non-convex and ill-conditioned.
To circumvent this issue, we propose the \emph{convex relaxation}:
\[
u_{i-1} = Px_i - L_ix_0 -\sum_{k=1}^{i-1}T_ku_{i-1-k}
\]
Here, $P$ corresponds to $B^+$, $L_i$ to $B^+A^i$, and $T_k$ to $B^+A^kB$.
We arrive at \eqref{eq:linear_relaxation} by considering this inverse model over a trajectory of length $r$, which is chosen so that we can recover the entirety of $V$.

We emphasize that our learning objective rules out trivial representations.
The function that maps everything to $0$ is technicaly a linearizing representation, but this solution is obviously undesirable.
By learning an inverse model, we are requiring our state representations to retain enough information to recover the actions.
In this sense they encode the core dynamics that govern how the control inputs affect the state, thereby enabling policy learning based on the state representations.

For this approach to work, we need a few assumptions.
\begin{assumption}[No Linear Dependence] Let $h_0$ and $u_0, \ldots, u_{i-1}$ be independent standard Gaussian vectors. There is a constant $0 \leq \rho < 1$ such that for each $i = 1,\ldots, r$, $\rho((z_i,z_0),(h_i,h_0)) \leq \rho$.
\label{assume:nonlinear}
\end{assumption}
We need assumption \ref{assume:nonlinear} to preclude any linear dependence between $z_i$ and the controls, otherwise the linear 
model that we learn may use information from $V^\perp$ to predict the controls. 

\begin{assumption}[Controllability]
\label{assume:controllability}
The tuple $(A^\top,(B^+)^\top)$ is $V$-controllable.
\end{assumption} 
Assumption \ref{assume:controllability} is related to the standard controllability assumption for linear control systems.
Instead of assuming $(A,B)$ controllability, we need the property to hold for $(A^\top, (B^+)^\top)$ since we are learning an inverse model which yields the matrices $B^+, B^+A, \ldots, B^+A^r$. This property holds when $A$ and $B$ are in general position.

\begin{assumption}[Non-degeneracy]
\label{assume:fullrank}
The matrix $B$ has linearly independent columns, i.e. $\text{rank}(B)=l$.
\end{assumption}
Assumption \ref{assume:fullrank} allows us to learn the inverse model. 
If $B$ is rank-deficient, we could not hope to predict even $u_0$ from $x_0$ and $x_1$, since it is non-identifiable. 
One interpretation of this assumption is that the control inputs $u_i$ are well-specified, i.e. not redundant.
The assumption holds if $B$ is in general position.

After the convex relaxation, the matrices $P$, $L_i$ and $T_i$ may not have a unique solution, as there are now some redundancies in the parametrization and in general the solution to a linear system can be a linear subspace. We show that one can still recover the intended solutions $B^+$, $B^+A_i$ and $B^+A^kB$ by imposing some norm preferences. We can now state the theoretical guarantee for our algorithm. 
\begin{theorem}
\label{thm:linear_regression_problem}
Let $f$ be the objection function in \eqref{eq:linear_relaxation}, and 
let $\Theta^*_0 = \{\theta = (P,\{L_i\}_{i=1}^r,\{T_i\}_{i=1}^{r-1}) \in f^{-1}(0)\,|\, \|P\|_F \text{ is minimal}\}$ be the set of optimal solutions to \eqref{eq:linear_relaxation} that have minimal norm for $P$. Let  $\theta^* = (P^*, \{L_i^*\}, \{T_i^*\}) \in \Theta^*_0$ be the solution in this set that minimizes $\sum_{i=1}^r \|L_i\|_F^2$. 
Then under  assumptions \ref{assume:nonlinear}, \ref{assume:controllability}, and \ref{assume:fullrank}, $P = B^+$ and $L_i = B^+A^i$ for $i = 1, \ldots, r$. 
Moreover, $V =  \text{col}(P^\top) + \text{col}(L_1^\top) + \cdots + \text{col}(L_r^\top)$.
\end{theorem}

To find the desirable solution, we can first find the set $\Theta^*$ for which the objective function is equal to $0$. As we discussed for such linear systems $\Theta^*$ is a subspace, $\Theta^*_0$ can then be obtained by optimizing for the norm of $P$ within this subspace. In supplementary material we also discuss how to find such a solution more efficiently. 

Intuitively, Theorem~\ref{thm:linear_regression_problem} is correct because by Assumption~\ref{assume:nonlinear}, any direction in the orthogonal subspace of $V$ will not have a perfect linear correlation with the signal $u_i$ that we are trying to predict. This does not mean that every optimal solution to Equation \eqref{eq:linear_relaxation} has components  only in $V$ - it is still possible that components in $V^\perp$ cancel each other. However, if any of the matrices $P, L_i, T_k$ have components in $V^\perp$, removing those components will reduce the norm of the matrices while not changing the predictive accuracy. Therefore the minimum norm solution must lie in the correct subspace. Finally, the fact that we recover the entire space relies on Assumption~\ref{assume:controllability}. The detailed proof is deferred to supplementary material.
In the supplementary material we also adapt this result to the case where there is noise in the linear dynamics.

\subsection{Sample Complexity}
We can't solve \eqref{eq:linear_relaxation} in practice since we only have access to finitely many samples of the system.
However, given enough samples, solving the empirical version of the problem allows us to robustly recover the model parameters with high probability. 
We consider the finite sample problem 
\begin{equation}
\label{eq:empirical_problem}
\underset{\theta}{\min} \frac{1}{2n}\sum_{i=1}^r \| PX_i - L_iX_0 -\sum_{k=1}^{i-1}T_kU_{i-1-k} - U_{i-1}\|_F^2
\end{equation}
Here, the columns of $X_i \in \mathbb{R}^{d\times n}, U_i \in \mathbb{R}^{l\times n}$ are i.i.d. copies of $x_i$ and $u_i$, respectively.

We introduce the following assumption that allows us to utilize quantitative concentration results, and then state the sample complexity result.
\begin{assumption}[Sub-Gaussianity]
\label{assume:subgaussian} 
There exists a constant $C > 0$ such that 
for each $i \in \{1,\ldots, r\}$, $P(|\langle q, \Sigma_{\xi_i\xi_i}^{-1/2}\xi_i\rangle| > t) \leq \exp(-Ct^2)$ for any unit vector $q$, where we
define $\xi_i := (z_i, z_0, h_i, h_0)$.
\end{assumption}

\begin{theorem}
\label{thm:sample_complexity}
Let $f$ be the objection function in \eqref{eq:empirical_problem}, and 
let $\Theta^*_0 = \{\theta = (P,\{L_i\}_{i=1}^r,\{T_i\}_{i=1}^{r-1}) \in f^{-1}(0)\,|\, \|P\|_F \text{ is minimal}\}$ be the set of optimal solutions to \eqref{eq:linear_relaxation} that have minimal norm for $P$. Let  $\theta^* = (P^*, \{L_i^*\}, \{T_i^*\}) \in \Theta^*_0$ be the solution in this set that minimizes $\sum_{i=1}^r \|L_i\|_F^2$. 
Under assumptions \ref{assume:nonlinear}, \ref{assume:controllability}, \ref{assume:fullrank}, and \ref{assume:subgaussian}, there exists a constant $C_0$ such that if $n \geq C_0(d+rl)\log r\log^2(d+rl)/(1-\rho)^2$, then with probability at least $0.99$,  $P = B^+$ and $L_i = B^+A^i$ for $i = 1, \ldots, r$. 
\end{theorem}


\subsection{Non-linear State Representation Learning}
\label{sec:nonlinear}
In this section, we take the algorithm and insights derived from our hidden subspace model and extend them to the setting in which there are no latent linear dynamics in the original state observations. 
In this case, we need to consider learning a nonlinear state mapping $\phi$, such as a convolutional neural network if the state observations are images. 
Our goal is still to learn a representation that linearizes the relevant dynamics, and so we extend our inverse model in the straightforward way:
\begin{equation}
\label{eq:nonlinear_objective}
\underset{\theta}{\min} \frac{1}{2}\mathbb{E}\sum_{i=1}^\tau\| P\phi(x_i) - L_i\phi(x_0) -\sum_{k=1}^{i-1}T_ku_{i-1-k}-u_{i-1}\|_2^2
\end{equation}
Although in practice it is unlikely to achieve $0$ loss for this optimization problem, we still verify that this is a principled approach where zero loss implies that $\phi$ gives a nontrivial linearization.

\begin{theorem} \label{thm:nonlinear}
Let $\phi, P, \{L_i, T_i\}, i=1,\ldots,\tau$ be optimal solutions to the optimization problem \eqref{eq:nonlinear_objective}, and assume that these parameters incur zero loss.
Define $V =\text{col}(P^\top) + \text{col}(L_1^\top) + \cdots + \text{col}(L_{\tau-1}^\top)$, and assume that $\text{col}(L_{\tau}^\top) \subset V$.
Let $Q$ be the 
projection matrix onto $V$. 
Then there exist matrices $A \in \mathbb{R}^{n\times n}$ and $B \in \mathbb{R}^{n\times l}$ such that for each $x$ and $u$,
\[
Q\phi(f(x,u)) = A Q\phi(x) + Bu.
\]
\end{theorem}
To get the final linearizing representation, we need $\phi$ followed by the projection $Q$.
This is a nontrivial linearization because, as before, it encodes the inverse dynamics, i.e. the control input can be predicted
given the initial and current state representations and previous control inputs.

Intuitively, as $\tau$ increases, by result of Theorem~\ref{thm:linear_regression_problem} we can expect that one learns a larger and larger subspace with linear dynamics. Theorem~\ref{thm:nonlinear} shows that as soon as the rank of this linear subspace stops increasing at a certain length $\tau$ of the trajectory, it already learns a linear dynamics. At a high level, using the fact that $L_{\tau}^\top$ is contained in the subspace learned in step $\tau-1$, we can formalize a linear relationship between states and the actions. While this linear relationship acts over a trajectory of length $\tau$, note that our objective contains trajectories of different lengths, and by combining them we can show that the dynamics is linear in the subspace $V$. The detailed proof is deferred to supplementary material. 

\section{Experiments}
\label{sec:experiments}
We are interested in the following questions: does our algorithm produce low-dimensional representations that admit simple and effective policies, and how many samples are needed?  We explore these questions by focusing on two standard continuous control tasks from OpenAI Gym \citep{brockman2016openai}: `Pendulum-v0' and `MountainCarContinuous-v0'. 
Our representation learning algorithm is built on PyTorch \citep{paszke2017automatic}, and our policy search algorithms use the Stable Baselines library \citep{stable-baselines}.
We follow the basic approach taken by \citet{lillicrap2015continuous} in working with pixel observations:
 modify the environments so that each action is repeated over three consecutive timesteps in the original environment, and concatenate the resultant observations. The full details of the experiments in this section -- as well as additional experimental results -- are found in the supplementary materials.



\paragraph{Representation and Policy Learning}
To train our state representations,
we optimize \eqref{eq:nonlinear_objective} where $\tau \in \{10,15\}$ and  $\phi$ is a simple neural network with two convolutional layers followed by one or two fully-connected layers with ReLu activations. We first train in a fully stochastic manner by drawing a new batch of samples from the environment at each step and continuing until convergence. This method is highly sample-inefficient but allows us to minimize the population loss. In settings where unlabeled environment samples are cheap but collecting reward labels is expensive, this approach can be reasonable. We also train representations on a smaller, fixed datasets to test whether high-quality representations can be learned with fewer samples. 

Recent works have observed that it is possible to learn effective linear policies for common locomotion control tasks based on the low-dimensional state observations \citep{mania2018simple,rajeswaran2017towards}. These works motivate the hypothesis that our low-dimensional state representations may admit
effective linear policies, in contrast to the highly nonlinear policies that are typically learned from direct pixel observations. 

After training our state representations, we use TRPO \citep{schulman2015trust}, a standard policy gradient method, to search over the space of linear policies.
We use the implementation provided by Stable Baselines with the default hyperparameters except for the learning rate, for which we try 7 different values (all performed similarly for both environments). We plot the learning curves of representative models found during training. 

\paragraph{Baselines}
Baseline 1 tests how efficiently standard RL algorithms can find effective policies directly from the raw pixel observations. We employ several standard RL algorithms implemented in Stable Baselines with default hyperparameters except for learning rate, which we tune.  
Many of these algorithms failed to find reasonable policies within the number of training steps provided, so we only report the best-performing results here.

For baseline 2 we train the same standard RL algorithms on the original, low-dimensional state variables of the environments using the tuned hyperparameters provided by RL Baselines Zoo \citep{rl-zoo}. Because these methods operate on the original low-dimensional states and use optimized hyperparameters, they provide a reasonable benchmark for good performance on the tasks. We again only display results for the top-performing algorithms. 

As a final strong baseline, we train Dreamer \citep{hafner2019dream}, which was recently shown to obtain excellent results on many visual control tasks. 
The learning curves for all approaches are shown in Figure \ref{fig:learning_curves}.

\begin{figure}[t]
\label{fig:learning_curves}
\begin{center}
\includegraphics[width=\textwidth]{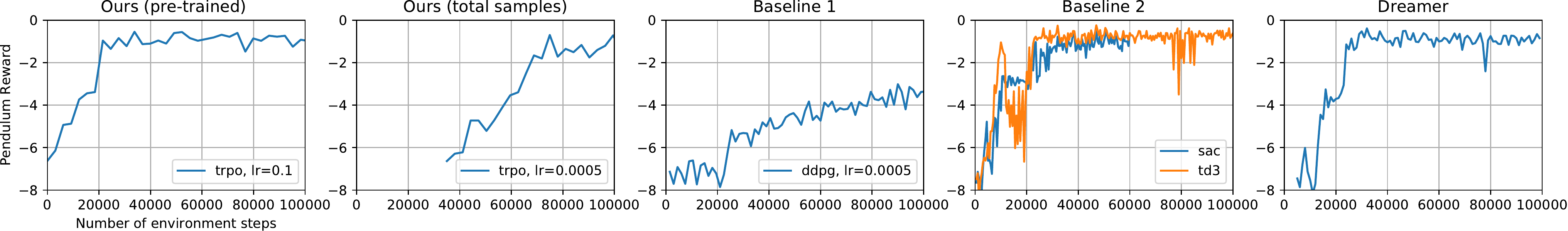}\\
\includegraphics[width=\textwidth]{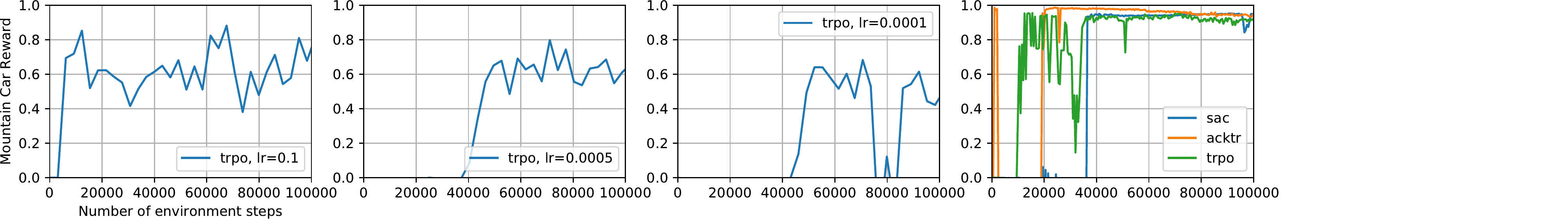}
\end{center}
\caption{Learning curves for `Pendulum-v0' (top) and `MountainCarContinuous-v0' (bottom).'}
\end{figure}

\paragraph{Results}
We first discuss the `Pendulum-v0' experiment. 
The leftmost plot (``pre-trained'') in Figure \ref{fig:learning_curves} shows the learning curve for TRPO \emph{after} training our representations on many  environment samples (greater than 100 thousand). Our learned representations, which trained only on the state dynamics with no knowledge of the reward function, allow TRPO to quickly find a good linear policy -- the performance is essentially the same as baseline 2 and Dreamer. Of course, the total sample complexity of this approach is much higher when we take into account the initial representation learning. 

The next plot (``total samples'') shows the learning curve for TRPO after training our representations on 35,000 samples, which is why the curve is shifted to the right. This plot indicates the overall sample-efficiency of our method -- it clearly outperforms baseline 1, but takes around 3 times as many samples as baseline 2 and Dreamer to reach similar performance. 


For the `MountainCarContinuous-v0' environment, positive reward indicates that the task is ``solved'', meaning the car has successfully reached the goal position at the top of the mountain. Once again the leftmost plot shows that our representations trained on a large amount of environment samples enable TRPO to quickly find a reasonable policy that guides the car to the goal position, albeit with less total reward than baseline 2.
The next plot shows the learning curve of TRPO after training our representations on 25,000 samples. In terms of overall sample efficiency and performance, our method is on par with baseline 1, although it obtains better reward performance. 
Dreamer failed to find good policies for this task (we did not attempt to tune hyperparameters).

We note that there is a reward gap between the high-dimensional learners (ours and baseline 1), and the low-dimensional baseline 2.  This may be partially explained by the fact that negative reward is given based on the magnitude of the actions taken. Since the high-dimensional agents are constrained to repeat each action three times they may be accumulated excessive negative reward. 

We conclude that our representations indeed enable effective linear control policies for these tasks, but there is room to further improve the total sample complexity of the representation learning and policy training.
Training low-dimensional state representations that encode the dynamics linearly via \eqref{eq:nonlinear_objective} can therefore be a reasonable approach for control with rich observations, particularly when unlabeled samples from the state dynamics are cheap.

\section{Conclusion and Future Work}
State representation learning is a promising way to bridge the complicated reinforcement learning problems and the simple linear models that have theoretical guarantees. In this paper we study a basic model for state representation learning and show that effective, low-dimensional state representations can be learned efficiently from rich observations. The algorithm inspired by our theory can indeed recover reasonable state representations for several tasks in OpenAI gym. There are still many open problems: the nonconvex objective \eqref{eq:nonlinear_objective} seems hard to optimize for the network architectures we tried; is there a way to design the architecture to make the nonconvex objective go to 0? The algorithm relies on a sample of initial state, which can be tricky for difficult problems;  can we complement our algorithm with an exploration strategy? Are there more realistic models for state representation learning that can also be learned efficiently? We hope our paper serves as a starting point towards these questions.

\bibliographystyle{plainnat}
\bibliography{staterep}

\begin{thebibliography}{37}
\providecommand{\natexlab}[1]{#1}
\providecommand{\url}[1]{\texttt{#1}}
\expandafter\ifx\csname urlstyle\endcsname\relax
  \providecommand{\doi}[1]{doi: #1}\else
  \providecommand{\doi}{doi: \begingroup \urlstyle{rm}\Url}\fi

\bibitem[Anand et~al.(2019)Anand, Racah, Ozair, Bengio, C{\^o}t{\'e}, and
  Hjelm]{anand2019unsupervised}
Ankesh Anand, Evan Racah, Sherjil Ozair, Yoshua Bengio, Marc-Alexandre
  C{\^o}t{\'e}, and R~Devon Hjelm.
\newblock Unsupervised state representation learning in atari.
\newblock In \emph{Advances in Neural Information Processing Systems}, pages
  8766--8779, 2019.

\bibitem[Brockman et~al.(2016)Brockman, Cheung, Pettersson, Schneider,
  Schulman, Tang, and Zaremba]{brockman2016openai}
Greg Brockman, Vicki Cheung, Ludwig Pettersson, Jonas Schneider, John Schulman,
  Jie Tang, and Wojciech Zaremba.
\newblock Openai gym.
\newblock \emph{arXiv preprint arXiv:1606.01540}, 2016.

\bibitem[Cohen et~al.(2018)Cohen, Hassidim, Koren, Lazic, Mansour, and
  Talwar]{cohen2018online}
Alon Cohen, Avinatan Hassidim, Tomer Koren, Nevena Lazic, Yishay Mansour, and
  Kunal Talwar.
\newblock Online linear quadratic control.
\newblock \emph{arXiv preprint arXiv:1806.07104}, 2018.

\bibitem[Du et~al.(2019{\natexlab{a}})Du, Kakade, Wang, and Yang]{du2019good}
Simon~S Du, Sham~M Kakade, Ruosong Wang, and Lin~F Yang.
\newblock Is a good representation sufficient for sample efficient
  reinforcement learning?
\newblock \emph{arXiv preprint arXiv:1910.03016}, 2019{\natexlab{a}}.

\bibitem[Du et~al.(2019{\natexlab{b}})Du, Krishnamurthy, Jiang, Agarwal,
  Dud{\'\i}k, and Langford]{du2019provably}
Simon~S Du, Akshay Krishnamurthy, Nan Jiang, Alekh Agarwal, Miroslav
  Dud{\'\i}k, and John Langford.
\newblock Provably efficient rl with rich observations via latent state
  decoding.
\newblock \emph{arXiv preprint arXiv:1901.09018}, 2019{\natexlab{b}}.

\bibitem[Fazel et~al.(2018)Fazel, Ge, Kakade, and Mesbahi]{fazel2018global}
Maryam Fazel, Rong Ge, Sham~M Kakade, and Mehran Mesbahi.
\newblock Global convergence of policy gradient methods for the linear
  quadratic regulator.
\newblock \emph{arXiv preprint arXiv:1801.05039}, 2018.

\bibitem[Finn et~al.(2016)Finn, Tan, Duan, Darrell, Levine, and
  Abbeel]{finn2016deep}
Chelsea Finn, Xin~Yu Tan, Yan Duan, Trevor Darrell, Sergey Levine, and Pieter
  Abbeel.
\newblock Deep spatial autoencoders for visuomotor learning.
\newblock In \emph{2016 IEEE International Conference on Robotics and
  Automation (ICRA)}, pages 512--519. IEEE, 2016.

\bibitem[Folkestad et~al.(2019)Folkestad, Pastor, Mezic, Mohr, Fonoberova, and
  Burdick]{folkestad2019extended}
Carl Folkestad, Daniel Pastor, Igor Mezic, Ryan Mohr, Maria Fonoberova, and
  Joel Burdick.
\newblock Extended dynamic mode decomposition with learned koopman
  eigenfunctions for prediction and control.
\newblock \emph{arXiv preprint arXiv:1911.08751}, 2019.

\bibitem[Gao et~al.(2019)Gao, Garber, Srebro, Wang, and
  Wang]{gao2019stochastic}
Chao Gao, Dan Garber, Nathan Srebro, Jialei Wang, and Weiran Wang.
\newblock Stochastic canonical correlation analysis.
\newblock \emph{Journal of Machine Learning Research}, 20\penalty0
  (167):\penalty0 1--46, 2019.

\bibitem[Ha and Schmidhuber(2018)]{ha2018world}
David Ha and J{\"u}rgen Schmidhuber.
\newblock World models.
\newblock \emph{arXiv preprint arXiv:1803.10122}, 2018.

\bibitem[Hafner et~al.(2019)Hafner, Lillicrap, Ba, and
  Norouzi]{hafner2019dream}
Danijar Hafner, Timothy Lillicrap, Jimmy Ba, and Mohammad Norouzi.
\newblock Dream to control: Learning behaviors by latent imagination.
\newblock \emph{arXiv preprint arXiv:1912.01603}, 2019.

\bibitem[Hardt et~al.(2018)Hardt, Ma, and Recht]{hardt2018gradient}
Moritz Hardt, Tengyu Ma, and Benjamin Recht.
\newblock Gradient descent learns linear dynamical systems.
\newblock \emph{The Journal of Machine Learning Research}, 19\penalty0
  (1):\penalty0 1025--1068, 2018.

\bibitem[Higgins et~al.(2017)Higgins, Pal, Rusu, Matthey, Burgess, Pritzel,
  Botvinick, Blundell, and Lerchner]{higgins2017darla}
Irina Higgins, Arka Pal, Andrei Rusu, Loic Matthey, Christopher Burgess,
  Alexander Pritzel, Matthew Botvinick, Charles Blundell, and Alexander
  Lerchner.
\newblock Darla: Improving zero-shot transfer in reinforcement learning.
\newblock In \emph{Proceedings of the 34th International Conference on Machine
  Learning-Volume 70}, pages 1480--1490. JMLR. org, 2017.

\bibitem[Hill et~al.(2018)Hill, Raffin, Ernestus, Gleave, Kanervisto, Traore,
  Dhariwal, Hesse, Klimov, Nichol, Plappert, Radford, Schulman, Sidor, and
  Wu]{stable-baselines}
Ashley Hill, Antonin Raffin, Maximilian Ernestus, Adam Gleave, Anssi
  Kanervisto, Rene Traore, Prafulla Dhariwal, Christopher Hesse, Oleg Klimov,
  Alex Nichol, Matthias Plappert, Alec Radford, John Schulman, Szymon Sidor,
  and Yuhuai Wu.
\newblock Stable baselines.
\newblock \url{https://github.com/hill-a/stable-baselines}, 2018.

\bibitem[Kawahara(2016)]{kawahara2016dynamic}
Yoshinobu Kawahara.
\newblock Dynamic mode decomposition with reproducing kernels for koopman
  spectral analysis.
\newblock In \emph{Advances in neural information processing systems}, pages
  911--919, 2016.

\bibitem[Lesort et~al.(2018)Lesort, D{\'\i}az-Rodr{\'\i}guez, Goudou, and
  Filliat]{lesort2018state}
Timoth{\'e}e Lesort, Natalia D{\'\i}az-Rodr{\'\i}guez, Jean-Franois Goudou, and
  David Filliat.
\newblock State representation learning for control: An overview.
\newblock \emph{Neural Networks}, 108:\penalty0 379--392, 2018.

\bibitem[Li et~al.(2017)Li, Dietrich, Bollt, and Kevrekidis]{li2017extended}
Qianxiao Li, Felix Dietrich, Erik~M Bollt, and Ioannis~G Kevrekidis.
\newblock Extended dynamic mode decomposition with dictionary learning: A
  data-driven adaptive spectral decomposition of the koopman operator.
\newblock \emph{Chaos: An Interdisciplinary Journal of Nonlinear Science},
  27\penalty0 (10):\penalty0 103111, 2017.

\bibitem[Lillicrap et~al.(2015)Lillicrap, Hunt, Pritzel, Heess, Erez, Tassa,
  Silver, and Wierstra]{lillicrap2015continuous}
Timothy~P Lillicrap, Jonathan~J Hunt, Alexander Pritzel, Nicolas Heess, Tom
  Erez, Yuval Tassa, David Silver, and Daan Wierstra.
\newblock Continuous control with deep reinforcement learning.
\newblock \emph{arXiv preprint arXiv:1509.02971}, 2015.

\bibitem[Lusch et~al.(2018)Lusch, Kutz, and Brunton]{lusch2018deep}
Bethany Lusch, J~Nathan Kutz, and Steven~L Brunton.
\newblock Deep learning for universal linear embeddings of nonlinear dynamics.
\newblock \emph{Nature communications}, 9\penalty0 (1):\penalty0 1--10, 2018.

\bibitem[Mania et~al.(2018)Mania, Guy, and Recht]{mania2018simple}
Horia Mania, Aurelia Guy, and Benjamin Recht.
\newblock Simple random search provides a competitive approach to reinforcement
  learning.
\newblock \emph{arXiv preprint arXiv:1803.07055}, 2018.

\bibitem[Misra et~al.(2019)Misra, Henaff, Krishnamurthy, and
  Langford]{misra2019kinematic}
Dipendra Misra, Mikael Henaff, Akshay Krishnamurthy, and John Langford.
\newblock Kinematic state abstraction and provably efficient rich-observation
  reinforcement learning.
\newblock \emph{arXiv preprint arXiv:1911.05815}, 2019.

\bibitem[Paszke et~al.(2017)Paszke, Gross, Chintala, Chanan, Yang, DeVito, Lin,
  Desmaison, Antiga, and Lerer]{paszke2017automatic}
Adam Paszke, Sam Gross, Soumith Chintala, Gregory Chanan, Edward Yang, Zachary
  DeVito, Zeming Lin, Alban Desmaison, Luca Antiga, and Adam Lerer.
\newblock Automatic differentiation in pytorch.
\newblock 2017.

\bibitem[Pathak et~al.(2017)Pathak, Agrawal, Efros, and
  Darrell]{pathak2017curiosity}
Deepak Pathak, Pulkit Agrawal, Alexei~A Efros, and Trevor Darrell.
\newblock Curiosity-driven exploration by self-supervised prediction.
\newblock In \emph{Proceedings of the IEEE Conference on Computer Vision and
  Pattern Recognition Workshops}, pages 16--17, 2017.

\bibitem[Qin(2006)]{qin2006overview}
S~Joe Qin.
\newblock An overview of subspace identification.
\newblock \emph{Computers \& chemical engineering}, 30\penalty0
  (10-12):\penalty0 1502--1513, 2006.

\bibitem[Raffin(2018)]{rl-zoo}
Antonin Raffin.
\newblock Rl baselines zoo.
\newblock \url{https://github.com/araffin/rl-baselines-zoo}, 2018.

\bibitem[Raffin et~al.(2019)Raffin, Hill, Traor{\'e}, Lesort,
  D{\'\i}az-Rodr{\'\i}guez, and Filliat]{raffin2019decoupling}
Antonin Raffin, Ashley Hill, Kalifou~Ren{\'e} Traor{\'e}, Timoth{\'e}e Lesort,
  Natalia D{\'\i}az-Rodr{\'\i}guez, and David Filliat.
\newblock Decoupling feature extraction from policy learning: assessing
  benefits of state representation learning in goal based robotics.
\newblock \emph{arXiv preprint arXiv:1901.08651}, 2019.

\bibitem[Rajeswaran et~al.(2017)Rajeswaran, Lowrey, Todorov, and
  Kakade]{rajeswaran2017towards}
Aravind Rajeswaran, Kendall Lowrey, Emanuel~V Todorov, and Sham~M Kakade.
\newblock Towards generalization and simplicity in continuous control.
\newblock In \emph{Advances in Neural Information Processing Systems}, pages
  6550--6561, 2017.

\bibitem[Schulman et~al.(2015)Schulman, Levine, Abbeel, Jordan, and
  Moritz]{schulman2015trust}
John Schulman, Sergey Levine, Pieter Abbeel, Michael Jordan, and Philipp
  Moritz.
\newblock Trust region policy optimization.
\newblock In \emph{International conference on machine learning}, pages
  1889--1897, 2015.

\bibitem[Shelhamer et~al.(2016)Shelhamer, Mahmoudieh, Argus, and
  Darrell]{shelhamer2016loss}
Evan Shelhamer, Parsa Mahmoudieh, Max Argus, and Trevor Darrell.
\newblock Loss is its own reward: Self-supervision for reinforcement learning.
\newblock \emph{arXiv preprint arXiv:1612.07307}, 2016.

\bibitem[Silver et~al.(2017)Silver, Schrittwieser, Simonyan, Antonoglou, Huang,
  Guez, Hubert, Baker, Lai, Bolton, et~al.]{silver2017mastering}
David Silver, Julian Schrittwieser, Karen Simonyan, Ioannis Antonoglou, Aja
  Huang, Arthur Guez, Thomas Hubert, Lucas Baker, Matthew Lai, Adrian Bolton,
  et~al.
\newblock Mastering the game of go without human knowledge.
\newblock \emph{Nature}, 550\penalty0 (7676):\penalty0 354--359, 2017.

\bibitem[Van~Hoof et~al.(2016)Van~Hoof, Chen, Karl, van~der Smagt, and
  Peters]{van2016stable}
Herke Van~Hoof, Nutan Chen, Maximilian Karl, Patrick van~der Smagt, and Jan
  Peters.
\newblock Stable reinforcement learning with autoencoders for tactile and
  visual data.
\newblock In \emph{2016 IEEE/RSJ International Conference on Intelligent Robots
  and Systems (IROS)}, pages 3928--3934. IEEE, 2016.

\bibitem[Vershynin(2010)]{vershynin2010introduction}
Roman Vershynin.
\newblock Introduction to the non-asymptotic analysis of random matrices.
\newblock \emph{arXiv preprint arXiv:1011.3027}, 2010.

\bibitem[Vinyals et~al.(2019)Vinyals, Babuschkin, Chung, Mathieu, Jaderberg,
  Czarnecki, Dudzik, Huang, Georgiev, Powell, et~al.]{vinyals2019alphastar}
Oriol Vinyals, Igor Babuschkin, Junyoung Chung, Michael Mathieu, Max Jaderberg,
  Wojciech~M Czarnecki, Andrew Dudzik, Aja Huang, Petko Georgiev, Richard
  Powell, et~al.
\newblock Alphastar: Mastering the real-time strategy game starcraft ii.
\newblock \emph{DeepMind blog}, page~2, 2019.

\bibitem[Watter et~al.(2015)Watter, Springenberg, Boedecker, and
  Riedmiller]{watter2015embed}
Manuel Watter, Jost Springenberg, Joschka Boedecker, and Martin Riedmiller.
\newblock Embed to control: A locally linear latent dynamics model for control
  from raw images.
\newblock In \emph{Advances in neural information processing systems}, pages
  2746--2754, 2015.

\bibitem[Williams et~al.(2015)Williams, Kevrekidis, and
  Rowley]{williams2015data}
Matthew~O Williams, Ioannis~G Kevrekidis, and Clarence~W Rowley.
\newblock A data--driven approximation of the koopman operator: Extending
  dynamic mode decomposition.
\newblock \emph{Journal of Nonlinear Science}, 25\penalty0 (6):\penalty0
  1307--1346, 2015.

\bibitem[Yeung et~al.(2019)Yeung, Kundu, and Hodas]{yeung2019learning}
Enoch Yeung, Soumya Kundu, and Nathan Hodas.
\newblock Learning deep neural network representations for koopman operators of
  nonlinear dynamical systems.
\newblock In \emph{2019 American Control Conference (ACC)}, pages 4832--4839.
  IEEE, 2019.

\bibitem[Zhang et~al.(2018)Zhang, Satija, and Pineau]{zhang2018decoupling}
Amy Zhang, Harsh Satija, and Joelle Pineau.
\newblock Decoupling dynamics and reward for transfer learning.
\newblock \emph{arXiv preprint arXiv:1804.10689}, 2018.

\end{thebibliography}
\newpage
\appendix
In this appendix, we first give proofs of the theoretical results stated in the main paper. Next, we discuss synthetic experiments that validate our theoretical results. Finally, we give the additional  details and experimental results to go along with Section \ref{sec:experiments} in the main paper.
\section{Deferred Proofs from Section~\ref{sec:model_detail}}
In this section we re-state and prove the theorems in the main paper. 
\subsection{Learning the Hidden Subspace Model}

We first give a proof of Theorem~\ref{thm:linear_regression_problem}, which shows that our objective function can recover the unknown subspace $V$. The theorem is restated below.

\begin{theorem}
Let $f$ be the objection function in \eqref{eq:linear_relaxation}, and 
let $\Theta^*_0 = \{\theta = (P,\{L_i\}_{i=1}^r,\{T_i\}_{i=1}^{r-1}) \in f^{-1}(0)\,|\, \|P\|_F \text{ is minimal}\}$ be the set of optimal solutions to \eqref{eq:linear_relaxation} that have minimal norm for $P$. Let  $\theta^* = (P^*, \{L_i^*\}, \{T_i^*\}) \in \Theta^*_0$ be the solution in this set that minimizes $\sum_{i=1}^r \|L_i\|_F^2$. 
Then under  assumptions \ref{assume:nonlinear}, \ref{assume:controllability}, and \ref{assume:fullrank}, $P = B^+$ and $L_i = B^+A^i$ for $i = 1, \ldots, r$. 
Moreover, $V =  \text{col}(P^\top) + \text{col}(L_1^\top) + \cdots + \text{col}(L_r^\top)$.
\end{theorem}

Recall the objective \eqref{eq:linear_relaxation} was:
\[
\underset{\theta}{\min} \frac{1}{2}\mathbb{E} \sum_{i=1}^r \| Px_i - L_ix_0 -\sum_{k=1}^{i-1}T_ku_{i-1-k} - u_{i-1}\|_2^2
\]

The main difficulty of proving this theorem lies in a mismatch between Assumption~\ref{assume:nonlinear} and our objective function \eqref{eq:linear_relaxation}: in the objective \eqref{eq:linear_relaxation}, we try to enforce a linear relationship between $x_i, x_0, u_1, u_2, ..., u_{i-1}$, while Assumption~\ref{assume:nonlinear} is about $(h_i, h_0)$ and $(z_i, z_0)$.
The following lemma helps relate the two.
\begin{lemma}
\label{lem:bounded_correlation} Let $i \in \{1,\ldots,r\}$.
Let $\tilde h_i = (h_i, h_0, u_0, \ldots, u_{i-2})$ and let $\tilde z_i = (z_i, z_0)$. Then
\[
\rho(\tilde h_i,\tilde z_i) \leq \rho((h_i,h_0),(z_i,z_0)).
\]
\end{lemma}
\begin{proof}
Note that the definition of $\tilde h_i$ doesn't make sense for $i=1$. In that case, define $\tilde h_1 = (h_1, h_0)$.
Let $u = (u_0, u_1, \ldots, u_{i-1})$.
Observe that the coordinates of $\tilde h_i$ are a subset of the coordinates of $(h_i, h_0, u)$, so
$\rho(\tilde h_i, \tilde{z}_i) \leq \rho((h_i,h_0,u), \tilde{z}_i)$.
Note that there exist matrices $P$ and $Q$ such that $h_i = Ph_0 + Qu$. 
Let $a_1, a_2 \in \mathbb{R}^r$, $b_1, b_2 \in \mathbb{R}^{d},$ and $a_3 \in \mathbb{R}^{il}$. 
Write $a_3 = Q^\top v_1 + v_2$, where $Qv_2 = 0$.
Note that $u$ is independent of $h_0$ and $\langle v_2, u\rangle$ is independent of each coordinate of $h_i$ (as these are Gaussian random vectors).
Then we have
\begin{align*}
\mathbb{E}[(\langle a_1, h_i\rangle + \langle a_2, h_0\rangle + \langle a_3, u\rangle)^2] &= \mathbb{E}[(\langle a_1, h_i\rangle + \langle a_2, h_0\rangle + \langle Q^\top v_1, u\rangle + \langle v_2, u\rangle + \langle P^\top v_1, h_0\rangle - \langle P^\top v_1, h_0\rangle)^2]\\
&= \mathbb{E}[(\langle a_1 + v_1, h_i\rangle + \langle a_2 - P^\top v_1, h_0\rangle + \langle v_2, u\rangle)^2]\\
&= \mathbb{E}[(\langle a_1 + v_1, h_i\rangle + \langle a_2 - P^\top v_1, h_0\rangle)^2] + \mathbb{E}[\langle v_2, u\rangle^2]
\end{align*}
Now $u$ is independent of $z_0$, so $\mathbb{E}[\langle v_2,u\rangle\langle b_2,z_0\rangle] = 0$. 
Moreover, $u$ and $z_i$ are conditionally independent given $h_i$, so we have
\begin{align*}
\mathbb{E}[\langle v_2,u\rangle(\langle b_1, z_i\rangle + \langle b_2, z_0\rangle)] &= \mathbb{E}[\langle v_2,u\rangle\langle b_1,z_i\rangle]\\
&= \mathbb{E}[\mathbb{E}[\langle v_2,u\rangle\langle b_1,z_i\rangle|h_i]]\\
&= \mathbb{E}[\mathbb{E}[\langle v_2,u\rangle|h_i]\mathbb{E}[\langle b_1,z_i\rangle|h_i]]\\
&= \mathbb{E}[\mathbb{E}[\langle v_2, u\rangle]\mathbb{E}[\langle b_1,z_i\rangle|h_i]]\\
&= 0.
\end{align*} 
Then we have
\begin{align*}
&\frac{\mathbb{E}[(\langle a_1, h_i\rangle + \langle a_2, h_0\rangle + \langle a_3, u\rangle)(\langle b_1, z_i\rangle + \langle b_2,z_0\rangle)]}{\sqrt{\mathbb{E}[(\langle a_1, h_i\rangle + \langle a_2, h_0\rangle + \langle a_3, u\rangle)^2]\mathbb{E}[(\langle b_1,z_i\rangle + \langle b_2,z_0\rangle)^2]}}\\ 
&\qquad\leq \frac{\mathbb{E}[(\langle a_1 + v_1, h_i\rangle + \langle a_2 - P^\top v_1, h_0\rangle)(\langle b_1, z_i\rangle + \langle b_2,z_0\rangle)}{\sqrt{\mathbb{E}[(\langle a_1 + v_1, h_i\rangle + \langle a_2 - P^\top v_1, h_0\rangle)^2]\mathbb{E}[(\langle b_1,z_i\rangle + \langle b_2,z_0\rangle)^2]}}\\
&\qquad\leq \rho((h_i,h_0), (z_i,z_0)).
\end{align*}
\end{proof}

We are ready to prove Theorem~\ref{thm:linear_regression_problem}:

\begin{proof}
The main idea of the proof is to derive conditions for the variables based on first-order optimality conditions. 
We first prove that the optimal variables have support only on the linearizing subspace. 
As a consequence, we can then show that these variables equal the true model parameters.

To start, fix $i \in \{1,\ldots,r\}$, define $\theta_i  = [P\,\,L_i\,\,T_1\,\cdots\,T_{i-1}]$, let $\tilde y_i = (y_i, -y_0,-u_{i-2},\ldots,-u_0)$, $\tilde h_i = (h_i, -h_0,-u_{i-2},\ldots,-u_0)$, and $\tilde z_i = (z_i,-z_0)$. 
Define $K = [I\,\, 0]^\top$ to be the block matrix that satisfies $\theta_iK= [P\,\,L_i]$. 
Define $\tilde V = \text{diag}(V, V, I, \ldots, I)$ to be the block diagonal matrix that satisfies $\tilde y_i = \tilde V \tilde h_i$, and note that $\tilde V$ has full column rank. 
Observe that there exists a matrix $M$ such that $u_{i-1} = M\tilde h_i$.

We can now express the objective function as  $f(\theta) = \sum_{i=1}^\tau f_i(\theta_i)$, where $f_i(\theta_i) = \frac{1}{2}\mathbb{E}\|\theta_i(K\tilde{z}_i +\tilde V \tilde{h}_i) - u_{i-1}\|_2^2$. Since each $f_i$ has minimal value $0$, any optimal point for $f$ must simultaneously optimize each $f_i$. Hence, $\nabla f(\theta_i) = 0$ is a necessary condition for optimality. 
To this end, we compute the gradient of $f_i$ as
\[
\nabla f_i(\theta_i) = \theta_i(K\Sigma_{\tilde{z}_i\tilde{z}_i}K^\top + \tilde V\Sigma_{\tilde{h}_i\tilde{h}_i}\tilde V^\top + \tilde V\Sigma_{\tilde{h}_i\tilde{z}_i}K^\top+K\Sigma_{\tilde{z}_i\tilde{h}_i}\tilde V^\top) - \Sigma_{u_{i-1}\tilde{z}_i}K^\top - \Sigma_{u_{i-1}\tilde h_i}\tilde V^\top
\]
We split the optimality condition according to orthogonal subspaces $V$ and $V^\perp$ to obtain
\begin{align}
\label{eq:grad1}
0 &= \theta_i(\tilde V\Sigma_{\tilde{h}_i\tilde{h}_i}\tilde V^\top + K\Sigma_{\tilde{z}_i\tilde{h}_i}\tilde V^\top) -  \Sigma_{u_{i-1}\tilde h_i}\tilde V^\top\\
\label{eq:grad2}
0 &= \theta_i(K\Sigma_{\tilde{z}_i\tilde{z}_i}K^\top + \tilde V\Sigma_{\tilde{h}_i\tilde{z}_i}K^\top)- \Sigma_{u_{i-1}\tilde{z}_i}K^\top
\end{align}
From \eqref{eq:grad1}, we have $\theta_i\tilde V\Sigma_{\tilde{h}_i\tilde{h}_i} =  \Sigma_{u_{i-1}\tilde h_i}- \theta_i K\Sigma_{\tilde{z}_i\tilde{h}_i}$, and plugging this into \eqref{eq:grad2} (while also clearing $K^\top$ by right-multiplying the equation by $K$) gives
\begin{align*}
0 &= \theta_iK\Sigma_{\tilde{z}_i\tilde{z}_i} + \theta_i\tilde V\Sigma_{\tilde{h}_i\tilde{h}_i}\Sigma_{\tilde{h}_i\tilde{h}_i}^+\Sigma_{\tilde{h}_i\tilde{z}_i} -  \Sigma_{u_{i-1}\tilde{z}_i}\\
&= \theta_iK\Sigma_{\tilde{z}_i\tilde{z}_i} - \theta_iK\Sigma_{\tilde{z}_i\tilde{h}_i}\Sigma_{\tilde{h}_i\tilde{h}_i}^+\Sigma_{\tilde{h}_i\tilde{z}_i} -  \Sigma_{u_{i-1}\tilde{z}_i} +  \Sigma_{u_{i-1}\tilde h_i}\Sigma_{\tilde{h}_i\tilde{h}_i}^+\Sigma_{\tilde{h}_i\tilde{z}_i}\\
&= \theta_iK(\Sigma_{\tilde{z}_i\tilde{z}_i} - \Sigma_{\tilde{z}_i\tilde{h}_i}\Sigma_{\tilde{h}_i\tilde{h}_i}^+\Sigma_{\tilde{h}_i\tilde{z}_i}) - M\Sigma_{\tilde h_i\tilde z_i} + M\Sigma_{\tilde h_i\tilde h_i}\Sigma_{\tilde h_i\tilde h_i}^+\Sigma_{\tilde h_i\tilde z_i}\\
&= \theta_iK\Sigma_{\tilde z_i\tilde z_i}(I - \Sigma_{\tilde z_i\tilde z_i}^+\Sigma_{\tilde{z}_i\tilde{h}_i}\Sigma_{\tilde{h}_i\tilde{h}_i}^+\Sigma_{\tilde{h}_i\tilde{z}_i}).
\end{align*}
By Lemma \ref{lem:bounded_correlation} and Assumption \ref{assume:nonlinear}, we have that $I - \Sigma_{\tilde z_i\tilde z_i}^+\Sigma_{\tilde{z}_i\tilde{h}_i}\Sigma_{\tilde{h}_i\tilde{h}_i}^+\Sigma_{\tilde{h}_i\tilde{z}_i}$ is nonsingular, so we conclude that $\theta_iK\Sigma_{\tilde z_i \tilde z_i} = 0$. 
In particular, this implies that $\theta_i K\Sigma_{\tilde z_i \tilde y_i} = 0$. 

We can now simplify \eqref{eq:grad1} as $0 = \theta_i\tilde V\Sigma_{\tilde h_i\tilde h_i}\tilde V^\top - \Sigma_{u_{i-1}\tilde h_i}\tilde V^\top = \theta_i \Sigma_{\tilde y_i\tilde y_i} - \Sigma_{u_{i-1}\tilde y_i}$.
This matrix equation can be naturally partitioned into blocks according to the block partition of $\theta_i$ and $\tilde y_i$.
Reading out the second block column gives $0 = -PA^i + L_iVV^\top$.
Reading out the $(k+1)$-st block column (for $k \geq 1$) gives $0 = -PA^kB + T_k$. 
The first block column gives
\begin{align*}
0 &= P\Sigma_{y_iy_i} - L_i\Sigma_{y_0y_i} -\sum_{k=1}^{i-1}T_k\Sigma_{u_{i-1-k}y_i} - \Sigma_{u_{i-1}y_i}\\
&= P(A^i(A^i)^\top + \sum_{k=1}^{i-1}A^kB(A^kB)^\top + BB^\top)  - L_i(A^i)^\top - \sum_{k=1}^{i-1}T_k(A^kB)^\top - B^\top\\
&= (PB-I)B^\top + (PA^i - L_i)(A^i)^\top + \sum_{k=1}^{i-1}(PA^kB - T_k)(A^kB)^\top\\
&= (PB-I)B^\top.
\end{align*}
Using Assumtion \ref{assume:fullrank}, we right-multiply by $(B^+)^\top$ to obtain $PB = I$. Since $P$ is the minimal-norm optimal solution, we conclude that $P = B^+$.
Then $L_iVV^\top = B^+A^i$ and $T_k = B^+A^kB$. Since we are also minimizing the norm of $L_i$, we see that $L_i$ must vanish on $V^\perp$, so that $L_i = L_iVV^\top$, and $L_i = B^+A^i$.
That we recover all of $V$ is a consequence of Assumption \ref{assume:controllability}.
\end{proof}

\subsection{Finite Sample Analysis}
In this section, for each random vector $z_i, y_i, h_i, u_i$ involved in the model, we consider corresponding sample matrices $Z_i, Y_i \in \mathbb{R}^{d\times n}, H_i \in \mathbb{R}^{r\times n}, U_i \in \mathbb{R}^{l\times n}$.
For sample covariance matrices we use the notation $\Sigma_{U_iU_i} = \frac{1}{n}U_iU_i^\top, \Sigma_{Y_iZ_i} = \frac{1}{n} Y_iZ_i^\top$, and so on. 

More precisely, let $H_0 \in \mathbb{R}^{r\times n}$ be a random matrix whose columns are independent standard Gaussian vectors.
Likewise, for $k = 0,\ldots,r-1$, let $U_k \in \mathbb{R}^{l\times n}$ be a matrix whose columns are independent standard Gaussian vectors. For $i = 0,\ldots, r$, let $H_i = \bar A^iH_0 + \sum_{k=0}^{i-1}\bar A^kBU_{i-1-k}$, and let $X_i = VH_i + Z_i$, where the columns of $X_i$ are the observed states and the columns of $Z_i$ are the nonlinear parts. 

Define empirical canonical correlation in the natural way: for random vectors $y$ and $z$, let $Y$ and $Z$ be the corresponding sample matrices and define
\[
\rho(Y,Z) = \max_{a,b}\frac{a^\top \Sigma_{YZ}b}{\sqrt{a^\top \Sigma_{YY}a}\sqrt{b^\top \Sigma_{ZZ} b}}.
\] 
Note that $\rho(Y,Z)^2$ is the largest singular value of $\Sigma_{ZZ}^+\Sigma_{ZY}\Sigma_{YY}^+\Sigma_{YZ}$.
Since our analysis of the optimization problem depends on such matrices being invertible, the main thing we need to show is that the empirical canonical correlation $\rho(Y,Z)$ is close to $\rho(y,z)$ when the number of samples is large enough. 
We utilize a concentration result stated in \citep{gao2019stochastic} that quantifies this. 
\begin{lemma}[Adapted from Corollary 7 of \cite{gao2019stochastic}]
\label{lem:empirical_corr}
Assume that $y \in \mathbb{R}^{k_1}$ and $z\in\mathbb{R}^{k_2}$ are sub-Gaussian, set $k = k_1+k_2$, and let $\epsilon \in (0,1)$. There exists a constant $C$  such that for any $t \geq 1$, if $n \geq Ct^2k\log^2k/\epsilon^2$ then $|\rho(Y,Z) - \rho(y,z)| \leq \epsilon$ with probability at least $1-\exp(-t^2k)$, 
\end{lemma}
Note that the statement of this result in \cite{gao2019stochastic} is slightly different since they don't specify the dependence of the sample complexity on the failure probability parameter $t$. Our version here is easily obtained by using Corollary 5.50 from \cite{vershynin2010introduction} to include the parameter $t$. 

We also need to ensure that a certain Gaussian empirical covariance matrix is invertible. We use the following standard matrix concentration inequality.
\begin{lemma}[From Corollary 5.35 of \cite{vershynin2010introduction}]
\label{lem:gaussian_concentration}
Let $Y \in \mathbb{R}^{k\times n}$ be a matrix whose entries are independent standard Gaussian random variables. Then for every $t \geq 0$, with probability at least $1-2\exp(-t^2/2)$ it holds that
\[
\sqrt{n}-\sqrt{k} - t \leq \sigma_{min}(Y).
\]
\end{lemma}

We now use these two concentration results to prove our main lemma for this section.
\begin{lemma}
\label{lem:sample_main}
Let $\tilde Z_i$ and $\tilde H_i$ be the sample matrices of $\tilde z_i$ and $\tilde h_i$, respectively (from the proof of Theorem~\ref{thm:linear_regression_problem}). 
Further define  $\hat H_i$ to be the sample matrix for the random vector $\hat h_i := (h_0,u_{i-1},u_{i-2},\ldots,u_0)$. 
Let $\mathcal{E}_i$ denote the event that $|\rho(\tilde H_i,\tilde Z_i) -\rho(\tilde h_i,\tilde z_i)| \leq (1-\rho)/2$.
Let $\mathcal{F}_i$ denote the event that $\sigma_{min}(\hat H_i) \geq 1/2$. 
There exists a constant $C_0$ such that if $n = C_0(d+rl)\log r\log^2(d+rl)/(1-\rho)^2$, then 
\[
P\left(\bigcap_{i=1}^r \mathcal{E}_i\cap \mathcal{F}_i\right) \geq 0.99.
\]
\end{lemma}
\begin{proof}
Set the failure probability parameter $t = C'\sqrt{\log r}$, where $C'$ is a large enough constant such that
\[
r(\exp(-t^2(2d+2r)) + 2\exp(-t^2/2)) \leq 0.01.
\]
Let $C$ be the constant from Lemma \ref{lem:empirical_corr} applied to $\tilde h_i$ and $\tilde z_i$ with $\epsilon = (1-\rho)/2$ -- we can take the same $C$ for each $i$ since we assume each $(h_i, z_i)$ satisfy the same sub-Gaussian property.
Set $C_0$ large enough so that when $n = C_0(d+rl)\log r\log^2(d+rl)/(1-\rho)^2$, the following hold for $i = 1, \ldots, r$:
\begin{align*}
n &\geq 4Ct^2(2d+2r+(i-2)l)\log^2(2d+2r + (i-2)l)/(1-\rho)^2,\\
\sqrt{n} &\geq 1/2 + \sqrt{r+(i-1)l} + t
\end{align*}

We first analyze $P(\mathcal{E}_i)$. Apply Lemma \ref{lem:empirical_corr} to $\tilde h_i \in \mathbb{R}^{2r+(i-2)l}$ and $\tilde z_i \in \mathbb{R}^{2d}$ with $\epsilon = (1-\rho)/2$ and the specified value of $t$. Then we see that $n$ is large enough to ensure that $P(\mathcal{E}_i) \geq 1-\exp(-t^2(2d+2r+(i-2)l)) \geq 1-\exp(-t^2(2d+2r))$.

Next, consider $P(\mathcal{F}_i)$. Apply Lemma \ref{lem:gaussian_concentration} to $\hat H_i$ with the specified value of $t$. 
Again it is clear that $n$ is large enough to ensure that $P(\mathcal{F}_i) \geq 1-2\exp(-t^2/2)$.

Finally, by the union bound, 
\begin{align*}
P\left(\bigcap_{i=1}^r \mathcal{E}_i\cap \mathcal{F}_i\right) &\geq 1-\sum_{i=1}^r(2-P(\mathcal{E}_i)+P(\mathcal{F}_i))\\
&\geq 1 - r(\exp(-t^2(2d+2r))+2\exp(-t^2/2))\\
&\geq 0.99.
\end{align*}
\end{proof}

We now prove Theorem \ref{thm:sample_complexity}.

\begin{proof}
Lemma \ref{lem:sample_main} provides the sample complexity and success probability -- all that's left is to analyze the empirical loss assuming that the conclusion of Lemma \ref{lem:sample_main} holds. 
Our analysis of the empirical loss is close to that of the population loss. 
We use the same notation as in the proof of Theorem~\ref{thm:linear_regression_problem}, e.g. $\tilde Y_i, \tilde H_i, \tilde U_i, \tilde Z_i$ are the sample matrices of $\tilde y_i, \tilde h_i, \tilde u_i, \tilde z_i$, respectively. Likewise, define $\theta_i$, $K$, and $\tilde V$ as before. 
We additionally define $\hat H_i$ to be the sample matrix for $(h_0,u_{i-1},u_{i-2},\ldots,u_0)$.

By the same argument as in the proof of Theorem~\ref{thm:linear_regression_problem}, we have that
\[
0 = \theta_i K\Sigma_{\tilde Z_i\tilde Z_i}(I-\Sigma_{\tilde Z_i\tilde Z_i}^+\Sigma_{\tilde Z_i\tilde H_i}\Sigma_{\tilde H_i\tilde H_i}^+\Sigma_{\tilde H_i\tilde Z_i}).
\]
The spectral norm of $-\Sigma_{\tilde Z_i\tilde Z_i}^+\Sigma_{\tilde Z_i\tilde H_i}\Sigma_{\tilde H_i\tilde H_i}^+\Sigma_{\tilde H_i\tilde Z_i}$ is $\rho(\tilde H_i,\tilde Z_i)$, and by assumption and Lemma \ref{lem:sample_main}, we have
\[\rho(\tilde H_i,\tilde Z_i)\leq \rho(\tilde h_i,\tilde z_i) + (1-\rho)/2 \leq (1+\rho)/2 < 1.
\] 
Hence, $(I-\Sigma_{\tilde Z_i\tilde Z_i}^+\Sigma_{\tilde Z_i\tilde H_i}\Sigma_{\tilde H_i\tilde H_i}^+\Sigma_{\tilde H_i\tilde Z_i})$ is robustly nonsingular, so we conclude that $\theta_i K\Sigma_{\tilde Z_i\tilde Z_i} = 0$ and likewise $\theta_i K \Sigma_{\tilde Z_i\tilde Y_i} = 0$.

Using this fact, we can continue to follow the proof of Theorem \ref{thm:linear_regression_problem} to obtain
\[
0 = \theta_i\tilde V\Sigma_{\tilde H_i\tilde H_i}\tilde V^\top - \Sigma_{U_{i-1}\tilde H_i}\tilde V^\top.
\]
Analyzing this equation is slightly more complicated now due to the fact that sample cross-covariance terms like $\Sigma_{H_0U_j}$ are nonzero (whereas the corresponding population covariances vanish due to independence). 
By splitting the equation into block columns, grouping terms, and simplifying the terms that cancel, it is straightforward to see that 
\[
0 = [(PA^i - L_i) \,\, (PB - I)\,\, (PAB-T_1)\,\,\cdots\,\,(PA^{i-1}B-T_{i-1})]\tilde V\Sigma_{\hat H_i\hat H_i}\tilde V^\top.
\]
By assumption, $\sigma_{min}(\hat H_i) \geq 1/2$, so $\Sigma_{\hat H_i\hat H_i}$ is robustly nonsingular. 
Hence, we have that
\[
0 = [(PA^i - L_i) \,\, (PB - I)\,\, (PAB-T_1)\,\,\cdots\,\,(PA^{i-1}B-T_{i-1})]\tilde V,
\]
which implies that $PB = I$ and $(PA^i-L_i)V = 0$ for all $i$. 
Since we assume $P$ has minimal norm, we conclude that $P = B^+$.
Thus, $L_iV = B^+A^iV$,  i.e. $L_i = B^+A^i$ on  the subspace $V$. By the construction of our minimal norm solution, we know that $L_i$ must vanish on $V^\perp$, and this completes the proof.
\end{proof}

\subsection{Handling Noise in the Model}
We now consider a simple version of our model with noise, and show that our algorithm identifies the correct subspace (up to an error proportional to the noise) in this setting as well. 
We consider a one-step trajectory where the initial state $x_0 = 0$, and we assume that our observation is corrupted by independent centered noise.
In particular, we can write the state observation as $x =  Bu + z + \xi$, where $\xi$ is a random vector in $\mathbb{R}^d$ that is independent of both $u$ and $z$.
Assume the noise covariance matrix $\Sigma_{\xi\xi}$ splits orthogonally along the subspace $V$ and $V^\perp$, that is, we can write
$\Sigma_{\xi\xi} = \Sigma_1 + \Sigma_2$, where $\Sigma_1$ is the covariance of the noise projected onto $V$ and $\Sigma_2$ is the covariance of the noise projected onto the column-span of $V^\perp$.
This orthogonal splitting is satisfied when $\xi$ is a spherical Gaussian random vector, for example.

Given this noisy state observation $x$ and control input $u$, the task is to recover the column-span of $B$ by learning a linear inverse model:
\begin{equation}
\label{eq:noisy_regression}
\underset{P}{\min} \frac{1}{2}\mathbb{E}_{u,\xi} \|Px - u\|_2^2
\end{equation}
Due to the noise term, this linear model will not achieve zero error.
However, we can bound the error of our solution as a function of the noise magnitude and correlation bound.
\begin{theorem}
\label{thm:noisy_regression}
Let $u \in \mathbb{R}^l$ and $\xi \in \mathbb{R}^d$ be independent spherical Gaussian random vectors, with $\Sigma_{\xi\xi} = \sigma^2I$. 
Let $P$ be the minimal norm optimal solution to the optimization problem \eqref{eq:noisy_regression}.
Write $P = P_1 + P_2$, where $P_1$ is the projection of $P$ onto $V$, and $P_2$ is its projection onto $V^\perp$.
In the noisy setting described above, we have $P_1 = B^+$ and 
\[
\|P_2\|_2 \leq \frac{\sigma\rho}{2\sqrt{1-\rho^2}}\|B^+\|_2\|P_1\|_2
\]
where $\sigma = \lambda_{max}(\Sigma_{\xi\xi})$ and $\rho := \rho(u,z)$.
\end{theorem}
Note that ideally we want $P_2 = 0$, since its rows are in $V^\perp$. 
This theorem says that the spectral norm of $P_2$ is small compared to $P_1$, 
which allows us to approximately recover $B^+$.

\begin{proof}
In this setting, the optimality conditions of \eqref{eq:noisy_regression} take the form
\begin{align}
\label{eq:opt1noise}
0 &= B\Sigma_{uu}(B^\top P_1 - I) + B\Sigma_{uz}P_2 + \sigma^2 P_1\\
\label{eq:opt2noise}
0 &= \Sigma_{zu}(B^\top P_1 - I) + \Sigma_{zz}P_2 + \sigma^2 P_2,
\end{align}	
Multiplying \eqref{eq:opt1noise} by $\Sigma_{zu}(B\Sigma_{uu})^+$ and subtracting \eqref{eq:opt2noise} yields the following identity (after simplification):
\begin{equation}
\label{eq:opt3noise}
(\sigma^2I+ \Sigma_{zz}-\Sigma_{zu}\Sigma_{uu}^{-1}\Sigma_{uz})P_2 = \sigma^2\Sigma_{zu}(B\Sigma_{uu})^+P_1.
\end{equation}

Let $Q_z$ be the (orthogonal) projection onto the column-span of $\Sigma_{zz}$, and note that we can write $Q_z = (\Sigma_{zz}^{1/2})^+\Sigma_{zz}^{1/2}$. 
Define $C = Q_z - (\Sigma_{zz}^{1/2})^+\Sigma_{zu}\Sigma_{uu}^{-1}\Sigma_{uz}(\Sigma_{zz}^{1/2})^+$. Note that $(\Sigma_{zz}^{1/2})^+\Sigma_{zu}\Sigma_{uu}^{-1}\Sigma_{uz}(\Sigma_{zz}^{1/2})^+$ has maximal eigenvalue $\rho^2$. Then $C$ has column-span equal to that of $\Sigma_{zz}$, with minimal nonzero singular value equal to $1-\rho^2$.

Set $\Gamma = (\sigma^{-1}\Sigma_{zz}^{1/2}C^{1/2})^+ + (\sigma^{-1}\Sigma_{zz}^{1/2}C^{1/2})^\top$.
Based on the properties of $C$ that we established, it is evident that $\Gamma$ has column-span equal to that of $\Sigma_{zz}$, and it has minimal singular value bounded below by $2$ by Lemma \ref{lem:lowest_sval_bound}. 
We have
\begin{align*}
P_2 &= (C^{1/2}\Gamma)^+C^{1/2}\Gamma P_2\\
 &= (C^{1/2}\Gamma)^+\sigma^{-1}(\Sigma_{zz}^{1/2})^+\left(\sigma^2 I + \Sigma_{zz} - \Sigma_{zu}\Sigma_{uu}\Sigma_{uz}\right)P_2 \\
&= (C^{1/2}\Gamma)^+\sigma^{-1}(\Sigma_{zz}^{1/2})^+\sigma^2\Sigma_{zu}\Sigma_{uu}^{-1}B^+P_1\\
&= \sigma\Gamma^+(C^{1/2})^+((\Sigma_{zz}^{1/2})^+\Sigma_{zu}\Sigma_{uu}^{-1/2})\Sigma_{uu}^{-1/2}B^+P_1.
\end{align*}
Now $(\Sigma_{zz}^{1/2})^+\Sigma_{zu}\Sigma_{uu}^{-1/2}$ must have maximal singular value equal to $\rho$, since it gives a symmetric low-rank factorization of $(\Sigma_{zz}^{1/2})^+\Sigma_{zu}\Sigma_{uu}^{-1}\Sigma_{uz}(\Sigma_{zz}^{1/2})^+$.
Hence, we finally have the bound
\[
\|P_2\|_2 \leq \frac{\sigma\rho}{2\sqrt{1-\rho^2}} \|\Sigma_{uu}^{-1/2}B^+\|_2\|P_1\|_2.
\]

\end{proof}	
	
\begin{lemma}
\label{lem:lowest_sval_bound}
For any matrix $A$, the minimal nonzero singular value of $A^+ + A^\top$ is at least 2.
\end{lemma}
\begin{proof}
Write the compressed SVD of $A^+$ as $U\Sigma V^\top$, and note that we can write $A^\top = U\Sigma^{-1}V^\top$.
It is then evident that the non-zero singular values of $A^+ + A^\top$ are of the form $x + x^{-1}$ for $x > 0$.
But $x+x^{-1} \geq 2$ for all $x > 0$.
\end{proof}

\subsection{Nonlinear Model}
We here give a proof of Theorem \ref{thm:nonlinear}.
For this section, instead of writing $\phi(x)$ to denote the state representation of $x$, we simply drop explicit reference to $\phi$ and agree that any system state we discuss has already been mapped to its representation via $\phi$. This will simplify notation but doesn't change any of the analysis. 

\begin{theorem}
Let $\phi, P, \{L_i, T_i\}, i=1,\ldots,\tau$ be optimal solutions to the optimization problem \eqref{eq:nonlinear_objective}, and assume that these parameters incur zero loss.
Define $V =\text{col}(P^\top) + \text{col}(L_1^\top) + \cdots + \text{col}(L_{\tau-1}^\top)$, and assume that $\text{col}(L_{\tau}^\top) \subset V$.
Let $Q$ be the orthogonal projection matrix onto $V$. 
Then there exist matrices $A \in \mathbb{R}^{n\times n}$ and $B \in \mathbb{R}^{n\times l}$ such that 
\[
Qf(x,u) = Qf(Qx,u) = AQx + Bu.
\]
\end{theorem}

\begin{proof}
Zero loss in the objective function implies that
\[
Pf(x,\{u_0,\ldots,u_{i-1}) = L_i x + \sum_{k=0}^{i-2}T_{i-1-k}u_k + u_{i-1}
\]
 for all $x \in \phi(\mathbb{R}^d)$ and $u_j \in \mathbb{R}^l$, $j = 0,\ldots,i-1$. 

Fix $2 \leq i \leq \tau+1$.
By assumption,
\[
Pf(x, \{u_0,\ldots,u_{i-1}\}) = L_ix + \sum_{k=0}^{i-2}T_{i-1-k}u_k + u_{i-1}.
\] 
But we can also express this as follows:
\begin{align*}
Pf(x, \{u_0,\ldots,u_{i-1}\}) &= Pf(f(x,u_0),\{u_1,\ldots,u_{i-1}\})\\
&= L_{i-1}f(x,u_0) + \sum_{k=0}^{i-3}T_{i-2-k}u_{k+1} + u_{i-1}
\end{align*}
Equating these two expressions and eliminating like terms gives
\[
L_{i-1}f(x,u_0) = L_i x + T_{i-1}u_0.
\]

Note that here it is crucial that we couple the $T_i$ matrices. Without the coupling we would not be able to eliminate the terms relating $u_i$ for $i>0$.

Next, let $\{v_1, \ldots, v_r\}$ be an orthonormal basis for $V$. Then we can write $Q = \sum_{j=1}^r v_jv_j^\top$.
Furthermore, by construction, for each $v_j$, there exist vectors $y_{j,0}, y_{j,1}, \ldots, y_{j,\tau}$ such that $v_j = P^\top y_{j,0}+ \sum_{i=1}^\tau L_i^\top y_{j,i}$.
Notice that for $i=1,\ldots,\tau+1$, since $\text{col}(L_{i}^\top) \subset V$, it holds that $L_i = L_i Q$. 
Then we have 
\begin{align*}
Qf(x,u) &= \sum_{j=1}^r v_jv_j^\top f(x,u)\\
&= \sum_{j=1}^rv_j\left(y_{j,0}^\top Pf(x,u) + \sum_{i=1}^\tau y_{j,i}^\top L_{i}f(x,u)\right) \\
&= \sum_{j=1}^r v_j\left(y_{j,0}^\top (L_1x +u) + \sum_{i=1}^\tau y_{j,i}^\top (L_{i+1}x + T_i u)\right)\\
&= \left(\sum_{j=1}^r\sum_{i=0}^\tau v_jy_{j,i}^\top L_{i+1}\right)Qx + \left(\sum_{j=1}^r\sum_{i=0}^\tau v_jy_{j,i}^\top T_i\right)u
\end{align*}
where we let $T_0 = I$.
Now set $A =\sum_{j=1}^r\sum_{i=0}^\tau v_jy_{j,i}^\top L_{i+1}$ and $B = \sum_{j=1}^r\sum_{i=0}^\tau v_jy_{j,i}^\top T_i$, and we have our result.
\end{proof}

\section{Synthetic Experiments}
In this section we discuss in detail how to obtain the particular minimal-norm solution to \eqref{eq:empirical_problem} that we require in Theorem \ref{thm:sample_complexity}. We then discuss synthetic numerical experiments that we conducted to validate the correctness of this result.

\subsection{Constructing the Solution} It simplifies things to consider the least squares problem
\[
\min_{x,y} \|Ax + By - c\|_2^2,
\]
where $A$ and $B$ are arbitrary matrices and $c$ is an arbitrary vector. Assume the space of solutions $\{x,y\}$ that have zero error is nonempty (i.e. it is an entire linear space of solutions).
We want to select the optimal solution $(x^*, y^*)$ such that for any other optimal solution $(x', y')$, we have $\|x^*\|_2 \leq \|x'\|_2$ and if $\|x^*\|_2 = \|x'\|_2$ then $\|y^*\|_2 \leq \|y'\|_2$.

We can obtain such a solution by splitting the problem into two stages. First, let $x^*$ be the minimal norm solution of 
\[
\min_x \|(I-P_B)Ax - (I-P_B)c\|_2^2,
\]
where $P_B$ is the orthogonal projection onto the column-span of $B$.
We can compute $x^*$ using standard least squares techniques such as using the singular value decomposition.
Then, let $y^*$ be the minimal norm solution to
\[
\min_y \|By - P_Bc + P_BAx^*\|_2^2.
\]

Let us verify that $(x^*, y^*)$ has the desired properties. 
Let $(x', y')$ be any solution, i.e. $Ax' + By' = c$. 
Left-multiplying the equation by $I-P_B$, we see that $(I-P_B)Ax' = (I-P_B)c$. 
By construction, we have that $\|x^*\|_2 \leq \|x'\|_2$. 
Now assume that $\|x^*\|_2 = \|x'\|_2$. 
This implies that $x^* = x'$ (the minimum-norm solution is unique).
Then we have $By' = P_B c - P_BAx' = P_Bc - P_BAx^*$.
Again, by construction we have that $\|y^*\|_2 \leq \|y'\|_2$, as desired.

\subsection{Numerical Verification}
To numerically validate our theoretical result, we generated system matrices $\bar{A}, \bar{B}$ at random (with i.i.d. Guassian entries) and multiplied $\bar{A}$ by a constant to ensure it is well-conditioned (to avoid numerical issues). We generated the nonlinear components $z_i$ either as independent Guassian nosie or low-degree polynomials of $h_i$. We collected $5(d+rl)$ samples for each run (this is lower than the sample complexity we give in the theorem, but it sufficed for our experiments). We then constructed the solution using the two-step procedure described above, using the built-in ``lstsq'' function in SciPy. We then checked that our constructed solution matched the solution guaranteed by Theorem \ref{thm:sample_complexity}. 
In all of our runs, whenever the computations were numerically stable, we recovered the expected solution. 
\section{Experimental Details for Section~\ref{sec:experiments}}
We here provide further details about the experiments discussed in Section~\ref{sec:experiments}.

As mentioned, we repeat the action three times for both environments, and the resulting concatenated pixel observations have sizes
$(64,192)$ and $(80, 360)$ for the pendulum and mountain car environments, respectively. Figure~\ref{fig:pixel_observations} displays examples of these state observations.

\begin{figure}[t]
\vskip 0.2in
\begin{center}
\centerline{\includegraphics[width=\columnwidth]{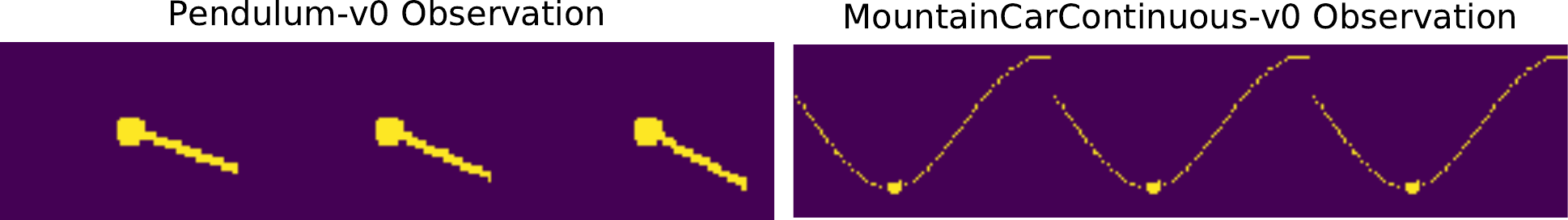}}
\caption{Pixel observations for the environments tested.}
\label{fig:pixel_observations}
\end{center}
\vskip -0.2in
\end{figure}

The state representation map $\phi$ is a basic neural network with two convolutional layers (each with 16 output channels, the first layer with kernel size 8 and stride 4, the second layer with kernel size 4 and stride 2) followed by two fully connected layers each of width $50$.
All layers use ReLu activation with no other nonlinearities.
After the final layer, we project to the top 4 right singular directions of the matrix $[P^\top \,\, L_1^\top \,\, \cdots \,\, L_\tau^\top]^\top$, so that in the end, we have a 4-dimensional representation.

To train the network $\phi$ and the matrices $P, \{L_i\}$, we solve \eqref{eq:nonlinear_objective} using the Adam optimizer with learning rate $0.0005$. The representations corresponding to the ``pre-trained'' results in Figure \ref{fig:learning_curves} were trained on 20,000 independent batches of 25 trajectories. Each trajectory was length 25, with the first state sampled uniformly from the environment state space, and each action sampled uniformly from the environment action space. We observed that the loss function converged to a nonzero value, which means there may be room to better learn the inverse model if we explore different architecture or training options. 

The representations corresponding to the ``total samples'' results in Figure \ref{fig:learning_curves} were trained on a fixed set of batches of trajectories that utilize a total of 35,000 and 25,000 environment steps for pendulum and mountain car, respectively. 
We found that training the representations on smaller amounts of environment steps led to poor performance on the control tasks, which may indicate that the representations overfit to the limited data.
Because we train these representations on a fixed, small dataset, the loss function converges to a value much closer to 0.

As mentioned in Section \ref{sec:experiments}, we use the Stable Baselines implementation of TRPO to learn a linear policy for our learned, 4-dimensional representations. We use all of the default parameters except for the stepsize parameter ``vf\_stepsize'', which we tested over the range of values $[0.00005, 0.0001, 0.0005, 0.001, 0.01, 0.1, 0.5]$. We observed similar performance for all of these choices, but reported the best results in Figure \ref{fig:learning_curves}, corresponding to $0.1$ and $0.0005$ for the pendulum experiments, and $0.1$ and $0.01$ for the mountain car experiments.

\section{Additional Experiments}
In this section we present some additional experimental results.

\paragraph{Intrinsic Reprentation Dimension}
In our main experiments, we projected our representations down to 4 dimensions. We inspected 
the intrinsic dimensionality of $[P^\top \,\, L_1^\top \,\, \cdots \,\, L_\tau^\top]$ and found that much of the energy was generally concentrated in the top four singular values. This agrees with our theoretical result, which requires that the learned matrices $P, \{L_i\}$ have low rank. Figure~\ref{fig:lowrank_rep} shows some examples of this behavior in representations learned for three environments (more on the ``HopperBulletEnv-v0'' environment below).

\begin{figure}[t]
\vskip 0.2in
\begin{center}
\centerline{\includegraphics[width=\columnwidth]{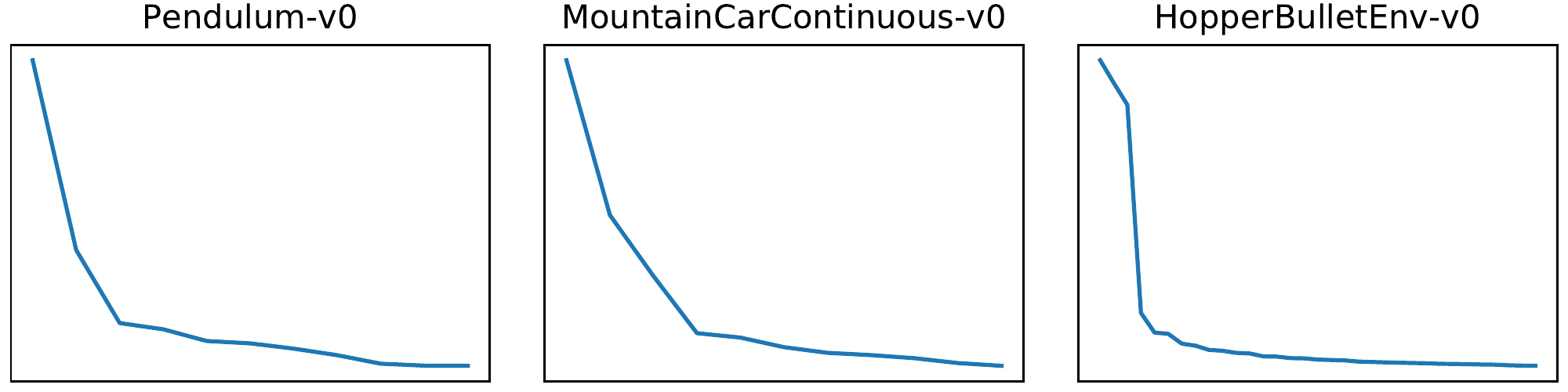}}
\caption{Singular values of learned representations.}
\label{fig:lowrank_rep}
\end{center}
\vskip -0.2in
\end{figure}

\paragraph{Experiments with Hopper}
We also conducted experiments using the environment ``HopperBulletEnv-v0'' from PyBullet. This environment simulates a robot whose internal state consists of 15 continuous values describing various joint angles and velocities, and whose action space is 3-dimensional, describing torques that can be applied to certain joints. The goal of this environment is to learn a policy that allows the Hopper robot to make forward progress. As with the other environments, we repeated each action 3 times to obtain concatenated pixel observations of shape $(80,321,3)$. Because the state space of this environment is much higher-dimensional than the other environments we tested, randomly sampling initial start states uniformly is a bad strategy, as most sampled states are very far from states that reasonable policies will traverse. 
We therefore adopted an imitation-learning style setup, where we use a pre-trained agent provided by Stable-Baselines to generate a set of good states, and use these examples as the initial states for our objective function. Given the good initial state, we then randomly sample actions to generate trajectories.
We trained our representations on a fixed set of 3,000 trajectories, and then projected onto the top 15 singular directions of the learned matrices, resulting in a 15-dimensional representation.

To learn policies, we used PPO2 with default parameters. The learned policies could keep the robot upright (it will fall to the ground if controlled with random actions) and sometimes achieve modest forward motion, but the performance is far below other baselines. To improve results, we likely need more training samples in the representation learning phase, and perhaps a way to combine state exploration with representation learning.

\end{document}